\documentclass{article}

\PassOptionsToPackage{numbers,compress}{natbib}
\usepackage[preprint]{neurips_2026}
\usepackage{fontawesome5}
\usepackage[utf8]{inputenc}
\usepackage[T1]{fontenc}
\usepackage{microtype}
\usepackage{hyperref}
\usepackage{url}
\usepackage{booktabs}
\usepackage{amsmath}
\usepackage{amssymb}
\usepackage{amsfonts}
\usepackage{nicefrac}
\usepackage{graphicx}
\usepackage{xcolor}
\usepackage{multirow}
\usepackage{subcaption}
\usepackage{enumitem}
\usepackage{algorithm}
\usepackage{algorithmic}
\usepackage{wrapfig}
\usepackage[breakable]{tcolorbox}
\usepackage{tcolorbox}
\definecolor{darkblue}{rgb}{0, 0, 0.5}
\hypersetup{colorlinks=true, citecolor=darkblue, linkcolor=darkblue, urlcolor=darkblue}

\newcommand{\ours}{G-Zero}

\newcommand{\ahard}{a_{\text{hard}}}
\newcommand{\aassisted}{a_{\text{assisted}}}

\newcommand{\eg}{\textit{e.g.}}

\usepackage{xspace}
\usepackage{amsmath}
\usepackage{graphicx}
\usepackage{framed}
\usepackage{bbm}
\usepackage{algorithm}
\usepackage{amsthm}
\usepackage[algo2e, vlined, ruled, linesnumbered]{algorithm2e}
\usepackage{natbib}
\usepackage{mathrsfs}
\usepackage{amssymb}
\usepackage{bm}
\usepackage{tabulary}
\usepackage{xcolor}
\usepackage{prettyref}
\usepackage{mathtools}
\usepackage{amsmath}

\usepackage{multirow}
\usepackage{nicefrac}
\usepackage{amssymb}
\usepackage{pifont}
\definecolor{Green}{rgb}{0.13, 0.65, 0.3}
\usepackage{colortbl}

\DeclareMathOperator*{\argmax}{argmax} 

\newtheorem{theorem}{Theorem}
\newtheorem{assumption}{Assumption}

\newtheorem{lemma}[theorem]{Lemma}

\usepackage{tikz}

\newcommand{\nonl}{\renewcommand{\nl}{\let\nl}}

\usepackage{prettyref}

\newcommand{\savehyperref}[2]{\texorpdfstring{\hyperref[#1]{#2}}{#2}}
\newrefformat{eq}{\savehyperref{#1}{Eq.~\textup{(\ref*{#1})}}}
\newrefformat{eqn}{\savehyperref{#1}{Equation~\ref*{#1}}}
\newrefformat{lem}{\savehyperref{#1}{Lemma~\ref*{#1}}}
\newrefformat{lemma}{\savehyperref{#1}{Lemma~\ref*{#1}}}
\newrefformat{def}{\savehyperref{#1}{Definition~\ref*{#1}}}
\newrefformat{line}{\savehyperref{#1}{Line~\ref*{#1}}}
\newrefformat{thm}{\savehyperref{#1}{Theorem~\ref*{#1}}}
\newrefformat{corr}{\savehyperref{#1}{Corollary~\ref*{#1}}}
\newrefformat{cor}{\savehyperref{#1}{Corollary~\ref*{#1}}}
\newrefformat{sec}{\savehyperref{#1}{Section~\ref*{#1}}}
\newrefformat{subsec}{\savehyperref{#1}{Section~\ref*{#1}}}
\newrefformat{app}{\savehyperref{#1}{Appendix~\ref*{#1}}}
\newrefformat{assum}{\savehyperref{#1}{Assumption~\ref*{#1}}}
\newrefformat{ex}{\savehyperref{#1}{Example~\ref*{#1}}}
\newrefformat{fig}{\savehyperref{#1}{Figure~\ref*{#1}}}
\newrefformat{alg}{\savehyperref{#1}{Algorithm~\ref*{#1}}}
\newrefformat{rem}{\savehyperref{#1}{Remark~\ref*{#1}}}
\newrefformat{conj}{\savehyperref{#1}{Conjecture~\ref*{#1}}}
\newrefformat{prop}{\savehyperref{#1}{Proposition~\ref*{#1}}}
\newrefformat{proto}{\savehyperref{#1}{Protocol~\ref*{#1}}}
\newrefformat{prob}{\savehyperref{#1}{Problem~\ref*{#1}}}
\newrefformat{claim}{\savehyperref{#1}{Claim~\ref*{#1}}}
\newrefformat{que}{\savehyperref{#1}{Question~\ref*{#1}}}
\newrefformat{op}{\savehyperref{#1}{Open Problem~\ref*{#1}}}
\newrefformat{fn}{\savehyperref{#1}{Footnote~\ref*{#1}}}
\newrefformat{tab}{\savehyperref{#1}{Table~\ref*{#1}}}
\newrefformat{fig}{\savehyperref{#1}{Figure~\ref*{#1}}}
\newrefformat{proc}{\savehyperref{#1}{Procedure~\ref*{#1}}}

\usepackage{caption}

\usepackage{accents}

\title{G-Zero: Self-Play for Open-Ended \\Generation from Zero Data}

\author{
  Chengsong Huang$^{1}$,~
  Haolin Liu$^{2}$,~
  Tong Zheng$^{3}$,~
  Runpeng Dai$^{4}$,~
  Langlin Huang$^{1}$, \\
  {\bfseries Jinyuan Li$^{1}$,~
  Zongxia Li$^{3}$,~
  Zhepei Wei$^{2}$,~
  Yu Meng$^{2}$,~
  Jiaxin Huang$^{1}$} \\
  \vspace{2pt}\\
  $^{1}$Washington University in St.\ Louis \quad
  $^{2}$University of Virginia \\
  $^{3}$University of Maryland \quad
  $^{4}$University of North Carolina at Chapel Hill \\
  \vspace{2pt}\\
  \texttt{\{chengsong,jiaxinh\}@wustl.edu}\vspace{-1.5em}
}

\begin{document}

\maketitle

\begin{abstract}
Self-evolving LLMs excel in verifiable domains but struggle in open-ended tasks, where reliance on proxy LLM judges introduces capability bottlenecks and reward hacking. To overcome this, we introduce \ours{}, a verifier-free, co-evolutionary framework for autonomous self-improvement. Our core innovation is \textbf{Hint-$\delta$}, an intrinsic reward that quantifies the predictive shift between a Generator model's unassisted response and its response conditioned on a self-generated hint. Using this signal, a Proposer model is trained via GRPO to continuously target the Generator's blind spots by synthesizing challenging queries and informative hints. The Generator is concurrently optimized via DPO to internalize these hint-guided improvements. Theoretically, we prove a best-iterate suboptimality guarantee for an idealized standard-DPO version of \ours{}, provided that the Proposer induces sufficient exploration coverage and the data filteration keeps pseudo-label score noise low. By deriving supervision entirely from internal distributional dynamics, \ours{} bypasses the capability ceilings of external judges, providing a scalable, robust pathway for continuous LLM self-evolution across unverifiable domains.
\end{abstract}

\vspace{-1.5em}
\begin{center}
\faGithub\ Code: \href{https://github.com/Chengsong-Huang/G-Zero}{\texttt{https://github.com/Chengsong-Huang/G-Zero}}
\end{center}
\vspace{-1.5em}
\section{Introduction}
\label{sec:intro}

Self-evolving Large Language Models (LLMs) have emerged as a promising path beyond the limits of human-curated supervision. Rather than relying on static datasets, these models autonomously generate, refine, and learn from their own outputs, offering a scalable route to capabilities that exceed what human imitation alone can provide~\citep{Tao2024ASO,Tan2024LargeLM,xiang2025systematic}. This potential has been most clearly demonstrated in reasoning-intensive tasks with strictly verifiable outcomes. In these settings, prior work~\citep{zhao2025absolute,huang2025r,liu2025spice} shows that models can discover complex problem-solving strategies through self-play, continuously improving toward expert-level performance.

However, this paradigm relies crucially on the existence of programmatic oracles. In domains like mathematics or code generation, deterministic signals, such as numerical correctness or functional execution, provide the ground truth required for Reinforcement Learning from Verifiable Rewards (RLVR)~\citep{Shao2024DeepSeekMathPT, DeepSeekAI2025DeepSeekR1IR}. Conversely, a broad class of real-world scenarios, including open-ended instruction following~\citep{shi2025hi}, multi-turn dialogue~\citep{yi2025survey}, and creative writing, lack such objective oracles. To navigate these settings, existing methods frequently rely on LLM-as-a-judge~\citep{gu2024survey} mechanisms for surrogate reward signals. This workflow introduces two critical limitations. First, the evolving model's performance ceiling is fundamentally bottlenecked by the judge's capabilities. Second, the optimization process is highly vulnerable to reward hacking~\citep{wang2026reward}; rather than genuinely improving response quality, the model learns to exploit the judge's stylistic vulnerabilities, such as bias, formatting preferences, or verbosity. This raises a crucial question: \emph{How can self-evolution scale in unverifiable domains without internalizing these pathologies?}
\begin{figure}
    \centering
    \includegraphics[width=0.95\linewidth]{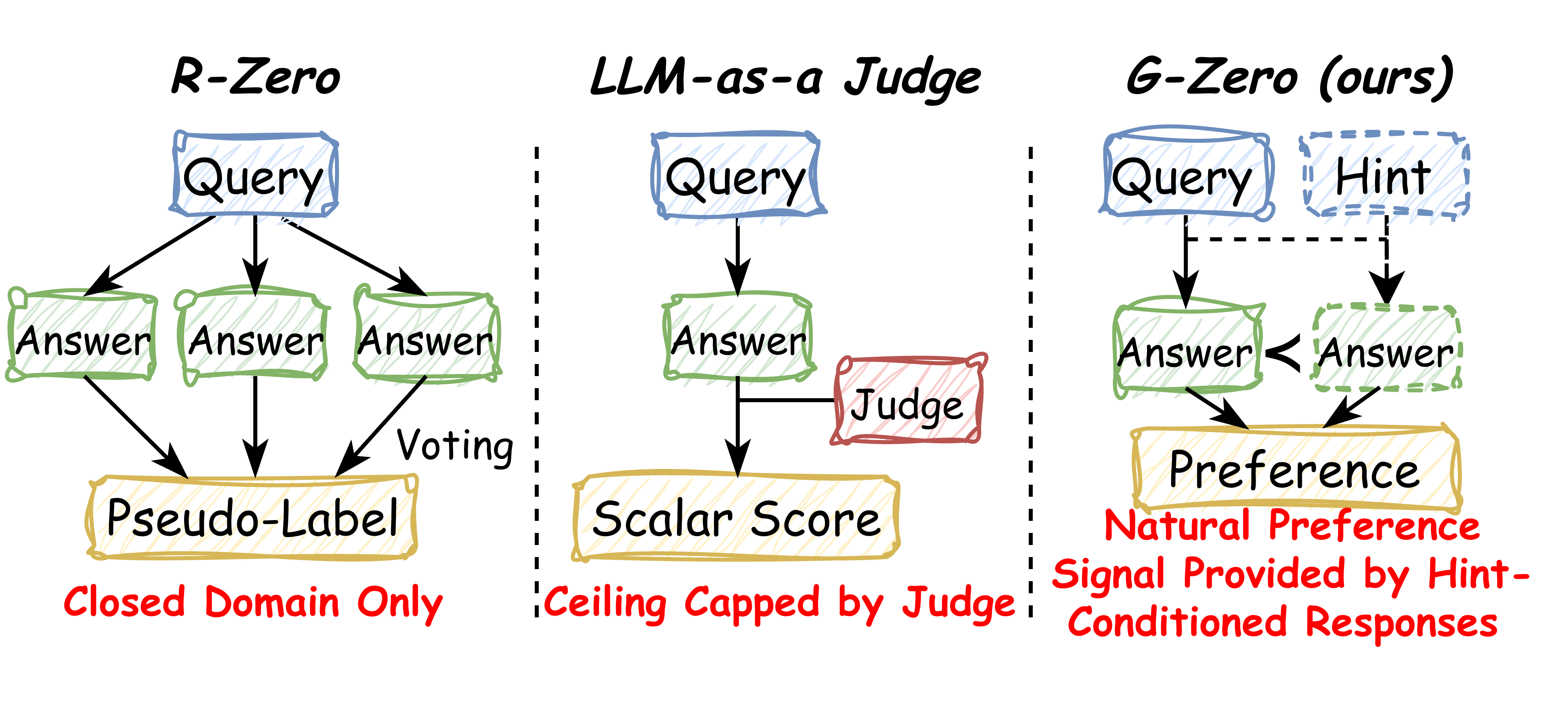}
    \caption{\textbf{Comparison of self-supervision signals.} R-Zero~\cite{huang2025r} uses majority voting, restricting it to verifiable closed-domain tasks. LLM-as-a-Judge assigns scalar scores, bounded by the judge's capability. In contrast, \ours{} creates an internal preference signal by preferring hint-conditioned responses over unassisted ones, eliminating the need for external verifiers or judges.}
    \label{fig:teaser}
\end{figure}

To move self-evolution beyond verifiable tasks and avoid the flaws of proxy LLM judges, we introduce \ours{}, a verifier-free, co-evolutionary framework that derives supervision entirely from internal dynamics. \ours{} operates through the interaction of two separate models: a Proposer and a Generator. 
The core innovation of \ours{} is our designed intrinsic signal called \textbf{Hint-$\delta$}, which measures how much a hint shifts the Generator's predictive distribution over its own unassisted response without the hint.
Hint-$\delta$ measures a cognitive gap by coupling two objectives into one scalar: it can be large only when the underlying query is challenging for the Generator \emph{and} the hint carries necessary information or reasoning that the Generator does not already possess. 
Using this intrinsic signal, the Proposer is trained via GRPO to synthesize challenging queries paired with informative hints, 
while the Generator is concurrently optimized via DPO to internalize these hint-guided improvements.
Specifically, the Generator learns to favor the hint-guided response (the \textit{chosen} output) over its initial, unassisted answer (the \textit{rejected} output). 
The two models co-evolve through iterative rounds.
This design directly addresses the two problems of judge-based self-evolution. Because Hint-$\delta$ is computed entirely from the Generator's own log-probabilities under matched contexts, the difficulty ceiling 
automatically improves with the Generator's capabilities.

We theoretically and empirically validate the \ours{} framework. 
Theoretically, we formalize the co-evolutionary loop and prove a best-iterate suboptimality guarantee for an idealized standard-DPO variant, under sufficient Proposer-induced coverage, and low \(\delta\)-certified pseudo-label score noise. Empirically, \ours{} demonstrates robust improvements within several self-play iterations, and achieves substantial gains on both open-ended (\eg, $+3.74$ points on AlpacaEval) and verifiable (\eg, $+5.21$ points in AIME 25) tasks across diverse model families (Qwen and Llama). Further analysis shows that the model's substantial reasoning improvements do not stem from domain-specific memorization, but from internalizing logical depth in open-ended, non-verifiable tasks, which surprisingly transfers to rigorous domains like mathematical problem-solving.

In summary, our main contributions are:
\begin{itemize}[leftmargin=*]
\setlength\itemsep{0em}
    \item \textbf{A Verifier-Free Co-Evolutionary Framework}: We propose \ours{}, a self-play pipeline that drives continuous self-evolution through hint-induced response shifts in open-ended domains without external verifiers.
    \item  \textbf{Theoretical Characterization of Intrinsic Self-Play}: We formalize the co-evolutionary loop and prove a best-iterate suboptimality guarantee for an idealized standard-DPO variant of \ours{}, with the bound controlled by Proposer-induced coverage, and \(\delta\)-certified pseudo-label score noise.
    \item \textbf{Empirical Improvements on Both Open-Ended and Verifiable Domains}: We demonstrate that \ours{} brings substantial improvements on instruction-following, chatting, and reasoning tasks across different model families, and also successfully internalize logical depth from open-ended, non-verifiable tasks to rigorous domains like mathematical problem-solving.
\end{itemize}

\section{Preliminaries: Optimization Objectives}
\label{sec:background}
\paragraph{Direct Preference Optimization (DPO).}
Direct Preference Optimization~\citep{rafailov2023direct} aligns a language model policy $\pi_\theta$ with preferences without requiring a separate reward model. Given a dataset $\mathcal{D}$ of preference triples $(x, y_w, y_l)$, where $x$ represents the input prompt, $y_w$ is the preferred (chosen) response, and $y_l$ is the rejected response, DPO optimizes the policy against a frozen reference model $\pi_{\text{ref}}$ by minimizing the following loss:
\begin{equation}
    \mathcal{L}_{\text{DPO}}(\theta) = -\,\mathbb{E}_{(x, y_w, y_l)\sim\mathcal{D}}\!\left[\log \sigma\!\left(\beta \log \frac{\pi_\theta(y_w \mid x)}{\pi_{\text{ref}}(y_w \mid x)} - \beta \log \frac{\pi_\theta(y_l \mid x)}{\pi_{\text{ref}}(y_l \mid x)}\right)\right].
    \label{eq:dpo}
\end{equation}
where $\sigma$ is the logistic function and $\beta$ is a hyperparameter that controls the deviation from the reference policy.

\paragraph{Group Relative Policy Optimization (GRPO).}
Group Relative Policy Optimization~\cite{Shao2024DeepSeekMathPT} is an efficient reinforcement learning algorithm that omits the need for an external value model. For a given context $c$ sampled from a dataset $\mathcal{P}$, the policy $\pi_\theta$ samples a group of $K$ outputs $\{o_1, \dots, o_K\}$ (we use $K$ rather than $G$ to avoid clashing with the Generator subscript $\pi_G$). The policy is updated by maximizing the following clipped objective, where $\epsilon \in (0, 1)$ is the PPO clip range:
\begin{equation}
\begin{aligned}
    \mathcal{J}_{\text{GRPO}}(\theta) = \mathbb{E}_{c \sim \mathcal{P}, \{o_i\}_{i=1}^K \sim \pi_{\text{old}}} \Bigg[ \frac{1}{K} \sum_{i=1}^K \min \bigg( & \frac{\pi_\theta(o_i \mid c)}{\pi_{\text{old}}(o_i \mid c)} A_i, \\
    & \text{clip}\Big(\frac{\pi_\theta(o_i \mid c)}{\pi_{\text{old}}(o_i \mid c)}, 1-\epsilon, 1+\epsilon\Big) A_i \bigg) \Bigg].
\end{aligned}
    \label{eq:grpo}
\end{equation}
Following prior work~\citep{liu2025understanding}, we omit the KL divergence penalty in our formulation. The advantage $A_i$ is computed by standardizing the scalar rewards $r(o_i)$ within the sampled group: $A_i = (r(o_i) - \mu_r) / \sigma_r$, where $\mu_r = \frac{1}{K} \sum_{j=1}^K r(o_j)$ and $\sigma_r = \sqrt{\frac{1}{K} \sum_{j=1}^K (r(o_j) - \mu_r)^2}$.
\section{The G-Zero Framework}
\label{sec:method}
\begin{figure}[t]
    \centering
    \includegraphics[width=0.95\linewidth]{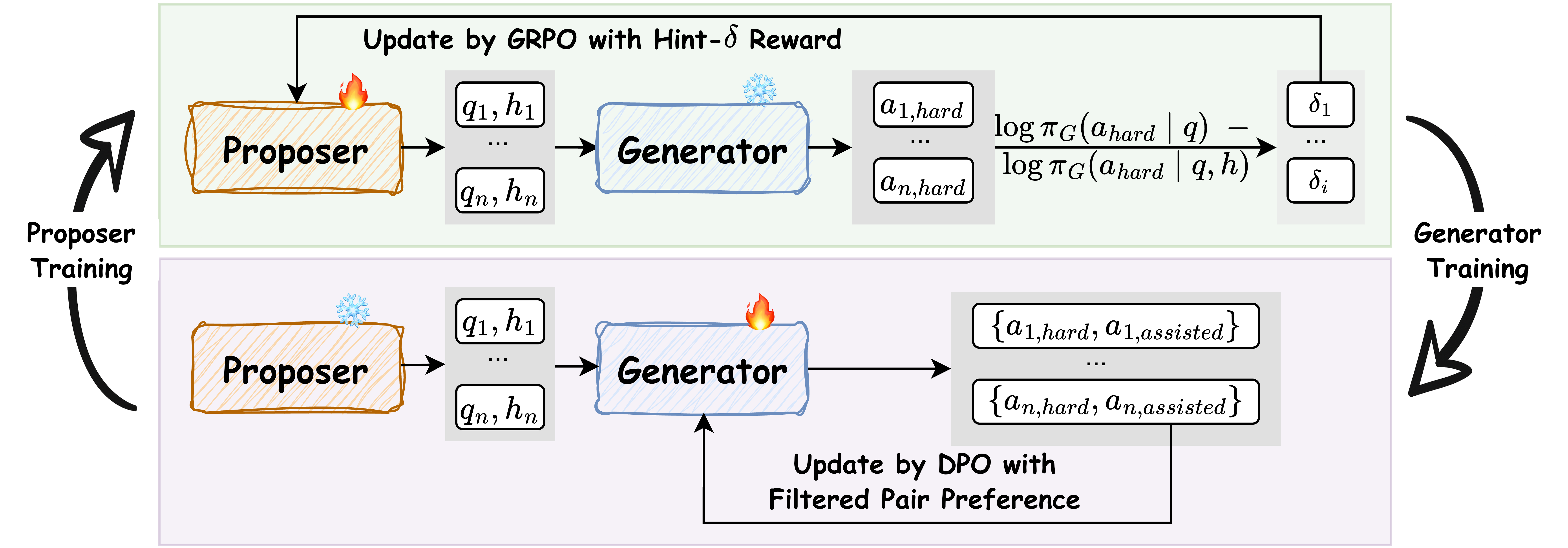}
    \caption{\textbf{The \ours{} co-evolutionary loop.}
    \textit{Top (Proposer training):} The Proposer $\pi_P$ generates query--hint pairs $\{(q_i, h_i)\}$. The frozen Generator $\pi_G$ produces unassisted responses, and Hint-$\delta$ is computed from the log-probability shift each hint induces on the Generator's distribution. The $\delta$ values serve as the GRPO reward, driving $\pi_P$ to explore the Generator's blind spots.
    \textit{Bottom (Generator training):} With $\pi_P$ frozen, the Generator answers each query with and without the hint; We then filter the resulting pairs into a preference dataset $\mathcal{D}_{R+1}$, where $R$ is the round index, on which $\pi_G$ is updated via DPO.}
    \label{fig:method}
\end{figure}
\ours{} is an iterative, co-evolutionary self-play framework designed for continuous LLM self-improvement. Instead of relying on external verifiers or inherently verifiable tasks, we construct preference pairs directly by contrasting the model's unassisted responses against those conditioned on intrinsic hints.

As illustrated in Figure~\ref{fig:method}, a single training round consists of two interacting phases: 
(1) \textbf{Proposer Training} (\S\ref{sec:phase1}): The Proposer is trained using Hint-$\delta$ (defined in \S\ref{sec:delta}) to identify challenging queries and pair them with informative hints. 
(2) \textbf{Dataset Curation and Generator Training} (\S\ref{sec:phase2}): Hint-$\delta$ is repurposed as a quality filter to curate well-suited response pairs. The Generator is then updated via DPO to favor the hint-guided responses over their unassisted baselines.

Through this iterative process, the Generator absorbs the structural and stylistic patterns elicited by the hints, and therefore learns to produce higher-quality independent responses. 
The improved model then serves as the base for the next round, enabling continuous self-evolution.

\subsection{The Intrinsic Learning Signal: Hint-$\delta$}
\label{sec:delta}

\ours{} is fundamentally driven by a single intrinsic learning signal, \textbf{Hint-$\delta$}.
Let $\pi_G$ denote the Generator LLM under training and $\pi_P$ the Proposer model used to explore the open-ended task space. 
For a given query $q$ and a proposed hint $h$, let $\ahard \sim \pi_G(\cdot \mid q)$ be the baseline response generated by the Generator without the hint, with token sequence $\ahard = (a_1, \dots, a_T)$. The Hint-$\delta$ signal measures how much the hint shifts the Generator's predictive distribution over its own unassisted response, evaluated as a \emph{per-token mean} log-likelihood difference:
\begin{equation}
    \delta(q, h, \ahard) \;=\; \frac{1}{T}\sum_{t=1}^{T}\Big[\log \pi_G(a_t \mid q, a_{<t}) - \log \pi_G(a_t \mid q, h, a_{<t})\Big].
    \label{eq:hintdelta}
\end{equation}
We deliberately use the per-token mean rather than the sequence-level sum so that $\delta$ is invariant to the length of $\ahard$: the Proposer cannot trivially inflate its reward by eliciting longer unassisted responses. Empirically, on our $1{,}824$-sample R1 raw pool we measure a Spearman rank correlation of $-0.41$ between $\delta$ and the character length of $\ahard$, i.e., longer $\ahard$ tend to receive \emph{smaller} $\delta$, which is consistent with the per-token normalization removing the naive length bias.

A key advantage of this formulation is that 
\textbf{Hint-$\delta$} effectively captures both query difficulty and hint informativeness at once, yielding a large $\delta$ only when two conditions are jointly met: (i) the underlying query is genuinely challenging for the Generator, so that the unassisted response is flawed or uncertain, and (ii) the hint carries missing knowledge or reasoning steps needed by the Generator to largely reshape the response distribution. 
Either factor alone is insufficient: If the query is trivial, the hints tend to be redundant with $\pi_G$'s prior knowledge, leaving the log-probability unchanged ($\delta \approx 0$). Symmetrically, for a difficult query, if the hint is uninformative, it fails to perturb the distribution. 
Consequently, maximizing $\delta$ drives the Proposer to jointly search over query difficulty and hint informativeness, automatically targeting the Generator's blind spots without any external difficulty signal. Crucially, the Proposer's reward is computed against the \emph{current} Generator, so as $\pi_G$ improves, the threshold for what counts as an ``informative hint'' rises with it. The two models therefore co-evolve across rounds.

\subsection{Proposer Training}
\label{sec:phase1}

The objective of this phase is to train the Proposer $\pi_P$ to actively propose challenging queries paired with informative hints that elicit a significant, constructive response shift in the Generator.

We design a specific system prompt that instructs $\pi_P$ to jointly generate a query $q$ and a corresponding hint $h$, enforcing a strict structural format utilizing \texttt{<question>} and \texttt{<hint>} XML tags. The Proposer is optimized via GRPO using Eq.~\eqref{eq:grpo}, where the output $o_i$ corresponds to the generated pair $(q_i, h_i)$. We use the Hint-$\delta$ signal in Eq.~\eqref{eq:hintdelta} as our intrinsic reward. 

However, optimizing purely for $\delta$ introduces vulnerabilities. A naive Proposer might learn to generate excessively verbose text to artificially shift the Generator's distribution. To prevent this reward hacking, we introduce a \textbf{Length Penalty}
\begin{equation}
    \mathcal{P}_{\text{length}} \;=\; \lambda \cdot \max\!\Big(0,\; \tfrac{|h|-200}{100}\Big),
\end{equation}
where $|h|$ is the hint length in characters and $\lambda = 0.03$ is used in all reported runs, penalizing hints that exceed a reasonable budget of $\sim$200 characters. Furthermore, to prevent the Proposer from collapsing into generating repetitive pairs, we apply a \textbf{BLEU Duplication Penalty} ($\mathcal{P}_{\text{BLEU}}$). We agglomeratively cluster all generated questions in the step's batch using sentence-BLEU distance with average-linkage and a merge threshold of $0.5$ (i.e., questions whose pairwise BLEU exceeds $0.5$ are merged into a cluster). For each rollout we set $\mathcal{P}_{\text{BLEU}} = |C_i| / |B|$, the fraction of the step's batch $B$ that lies in the rollout's own cluster $C_i$: a unique question receives a small $\sim\!1/|B|$ penalty, while a question shared by many rollouts is heavily discounted. 

The total reward combines the intrinsic $\delta$ signal with the penalties:
\begin{equation}
    r(q, h) \;=\; \delta(q, h, \ahard) \;-\; \mathcal{P}_{\text{length}} \;-\; \mathcal{P}_{\text{BLEU}}
    \label{eq:proposer_reward}
\end{equation}
For formatted error pairs (e.g., missing mandatory XML blocks or empty fields), we apply a hard-coded penalty floor of $-1$ and skip the $\delta$ computation entirely to save computation, while still applying the duplication penalty to punish repeated formatting failures.

\subsection{Generator Training and Dataset Curation}
\label{sec:phase2}
\label{sec:phase3}

In this final phase, we train the Generator $\pi_G$ on a curated preference dataset $\mathcal{D}_{R+1}$ using the DPO loss (Eq.~\eqref{eq:dpo}), with the hint-assisted response $\aassisted$ as the chosen sample ($y_w$) and the unassisted response $\ahard$ as the rejected sample ($y_l$). The reference model $\pi_{\text{ref}}$ is initialized as a frozen snapshot of $\pi_G$ taken at the start of the round, anchoring DPO updates to a stable behavioral baseline. To neutralize the well-known length bias of vanilla DPO, in which longer chosen responses contribute disproportionately to the gradient regardless of content, we adopt a length-normalized variant that replaces the sequence-summed log-ratio in Eq.~\eqref{eq:dpo} with its per-token mean:
\begin{equation}
    \mathcal{L}_{\text{DPO}}^{\text{LN}}(\theta) = -\,\mathbb{E}_{(x, y_w, y_l)\sim\mathcal{D}_{R+1}}\!\left[\log \sigma\!\left(\beta\bigl(\bar{r}_\theta(x, y_w) - \bar{r}_\theta(x, y_l)\bigr)\right)\right],
    \quad
    \bar{r}_\theta(x, y) = \frac{1}{|y|}\log\frac{\pi_\theta(y \mid x)}{\pi_{\text{ref}}(y \mid x)}.
    \label{eq:dpo_length_normalized}
\end{equation}
We adopt DPO rather than online RL with a learned reward model for two reasons. First, our preference pairs are constructed from the same model's output distribution under matched contexts, and DPO's closed-form, reference-anchored objective is a natural fit for this self-paired setting. Second, Hint-$\delta$ already provides an explicit chosen/rejected signal at the pair level; routing this signal through a separately trained reward model would introduce an additional information bottleneck and approximation error without any clear benefit.

The objective of this DPO training is hint internalization. By training on these pairs, the Generator is incentivized to favor the structural and stylistic patterns present in the hint-guided response, including more deliberate decomposition of the problem and more disciplined use of intermediate steps. As a result, the model tends to reproduce this higher-quality content independently, without requiring the explicit hint from the Proposer at inference time. This enables the Generator to perform substantially better on complex tasks during inference when no external assistance is available.
\paragraph{Training Set Curation.}
To maximize the efficacy of the DPO phase, we impose stringent filtering criteria on the preference pairs comprising $\mathcal{D}_{R+1}$. The Proposer's GRPO training, by maximizing Hint-$\delta$, has already performed a first stage of selection: the $(q, h)$ pairs it produces concentrate on hard queries equipped with informative hints. Our data curation performs a complementary selection on top of this pool to evaluate whether each pair is well-suited for DPO.

For each query-hint pair $(q, h) \sim \pi_P$, we sample the Generator's dual responses: the unassisted baseline $\ahard \sim \pi_G(\cdot \mid q)$ and the hint-conditional response $\aassisted \sim \pi_G(\cdot \mid q, h)$. We then recompute the $\delta$ score on these freshly sampled responses (Eq.~\eqref{eq:hintdelta}) and retain only pairs whose $\delta$ falls in the \emph{lower half} of the empirical distribution within each round.

While the Proposer targets the Generator's blind spots by maximizing $\delta$, we apply a contrasting filtering strategy for DPO data curation. In this stage, $\delta$ functions as a proxy for the distributional distance between $\aassisted$ and $\ahard$. Within the generated pool of complex queries, explicitly retaining preference pairs with relatively lower $\delta$ is essential, driven by two fundamental reasons:

\paragraph{Lower-$\delta$ pairs serve as hard-to-distinguish training signals.} In preference learning, training on pairs with a massive quality gap often yields diminishing returns, as the preference is trivially satisfied. A lower $\delta$ indicates that the log-probability shift between the chosen ($\aassisted$) and rejected ($\ahard$) responses is relatively minor. Consequently, these constitute hard-to-distinguish preference pairs. By focusing on these pairs where the reward gap is small, DPO is forced to learn fine-grained, structural improvements in reasoning rather than relying on superficial, easy-to-spot differences.

\paragraph{High-$\delta$ pairs violate DPO's implicit KL-divergence constraint.} The DPO formulation inherently includes a KL-divergence penalty against the reference model $\pi_{\text{ref}}$. A very high $\delta$ implies that the hint-assisted response $\aassisted$ is drastically far away from the Generator's original unassisted distribution. Pushing the policy towards such completely out-of-distribution responses severely violates this implicit KL constraint. This can lead to excessively large gradients, off-manifold drift, and severe training instability. By filtering out the top half of the $\delta$ distribution, we naturally regularize the optimization process and ensure $\aassisted$ remains a plausible trajectory for the Generator to internalize.

A subtle point concerns the very bottom of the retained band. Pairs with near-zero $\delta$ have low \emph{implicit-reward} margin under $\pi_G$, since the Generator assigns similar log-probabilities to $\aassisted$ and $\ahard$. We deliberately keep these pairs rather than excluding them so that the lower-half filter is defined purely by the ranked $\delta$ statistic without an additional minimum-margin cutoff: the bulk of the constructive learning signal in $\mathcal{D}_{R+1}$ is carried by the middle of the lower-half band, while the tail near zero adds only a small amount of low-margin label noise on responses that are independently sampled at temperature $0.7$ and therefore lexically distinct.

\subsection{Theoretical Analysis}

We analyze G-ZERO as an iterative $\delta$-certified exploratory DPO
procedure. We consider a simple linear case such that there exists a ground truth reward $R^\star(q,a) = \phi(q,a)^\top \theta^\star$ for any question $q$ and response $a$, where $\phi$ is a feature and $\theta^\star \in \mathbb{R}^d$ is a hidden reward parameter. We assume the standard Bradley-Terry model \cite{rafailov2023direct} such that $\mathbb{P}(a^+ \succ a^- |q) = \sigma\left(R^\star(q,a^+) - R^\star(q,a^-)\right)$. The performance of the generator has the following guarantee.
\begin{theorem} (Informal)
Suppose the game collects retained data from the Proposer such that, after $\delta$-filtering, the data are sufficiently exploratory, and the Generator is updated iteratively by DPO on the cumulative retained data for $T$ rounds. Then, with high probability, there exists an iterate $t_0\le T$ such that, using a total number of retained samples $\widetilde O(d^2/\varepsilon^2)$, the Generator's policy $\pi_{t_0}$ satisfies
\vspace{-2.7mm}

\[
J(\pi^\star)-J(\pi_{t_0})
\le
\widetilde O\!\left(
\varepsilon
+\sqrt{\eta_\delta}
\right).
\]

\vspace{-2.2mm}
where $\widetilde O$ omits $\log$ factors, 
$J(\pi)
=
\mathbb E_{q \sim Q, a\sim\pi(\cdot\mid q)}
[R^\star(q,a)]
-
\mathbb E_{q \sim Q }
\left[D_{\rm KL}\!\left(\pi(\cdot\mid q)\|\pi_{\mathrm{ref}}(\cdot\mid q)\right)\right]$, $Q$ is the target question distribution, $\pi^\star = \argmax_{\pi} J(\pi)$ and $\eta_\delta$ denotes the self-normalized cumulative score noise induced by incorrect pseudo-labels after $\delta$ filtration.
\label{thm:G-zero converge}
\end{theorem}
The theorem separates the two intrinsic signals in G-ZERO. The
Hint-$\delta$ and filter controls data quality: if retained pairs are calibrated so
that $a_{\rm assisted}$ is truly better than $a_{\rm hard}$ with high
probability, then the pseudo-label noise $\eta_\delta$ is small. The
exploration reward (implemented by BLEU) for the challenger controls data coverage: it drives the challenger toward
pairwise feature directions that the generator has not yet learned. Together,
these two effects imply an iterative co-evolution guarantee: the Challenger
supplies preference pairs that are both reliable enough to trust and novel
enough to teach, while cumulative DPO distills them into the Generator. Detailed proof of Theorem \ref{thm:G-zero converge} is provided in Appendix \ref{app:proof}.

\section{Experiments}
\label{sec:exp}

\subsection{Experimental Setup}
\label{sec:setup}

\paragraph{Models.}
To evaluate the generalization capabilities of our proposed method, we evaluate Qwen3-8B-Base~\citep{yang2025qwen3} and Llama-3.1-8B-Instruct~\citep{grattafiori2024llama}. By testing on both a foundational base model and an instruction-tuned model from distinct, widely adopted families, we demonstrate that our approach is robust to architectural variations and effective regardless of prior alignment stages.
\paragraph{Benchmarks and Evaluation.}
To evaluate reasoning capabilities, we benchmark on AIME24 and AIME25, reporting the overall mean@32 score from 32 independent responses sampled at a temperature of 0.7. To evaluate instruction-following, we use IFEval \citep{zhou2023instruction} with greedy decoding, reporting the four standard metrics (prompt/instruction-level strict and loose accuracies). Lastly, to assess general conversational quality, we report the length-controlled win rate on AlpacaEval 2.0 \citep{dubois2024length} against GPT-4-Turbo, judged by the Qwen3-235B-A22B-Instruct-2507.

\paragraph{Experiment Configuration.}
We strictly standardize the hyperparameter settings across the iterative loop. All model training in our experiments is conducted via the Tinker API~\footnote{https://thinkingmachines.ai/tinker/}, exclusively utilizing Low-Rank Adaptation (LoRA)~\citep{hu2022lora}.

We supplement the $\delta$-based filter with a set of lightweight heuristic checks on the chosen response $\aassisted$ to remove pairs that are known to induce DPO artifacts, following standard practice in DPO data curation. Specifically, to prevent the model from learning length as a proxy for quality, we discard pairs exhibiting length inflation ($l_w / l_l > 2.5$, where $l_w$ and $l_l$ are the character lengths of the chosen $\aassisted$ ($y_w$) and rejected $\ahard$ ($y_l$) responses). We also enforce absolute length bounds, requiring $l_w \in [100, 10000]$ to avoid degenerate gradients from extremely short or long responses. Furthermore, to prevent repetition collapse, we discard responses with a zlib compression ratio $< 0.15$, which reliably flags repetitive or degenerate text (highly repetitive sequences compress to a small fraction of their original size). Finally, we filter out instances of prompt echoing, discarding pairs where $\aassisted$ shares a prefix of $\ge 30$ characters with $q$, as well as template leakage, removing any responses containing raw role markers (e.g., ``Assistant:''). The remaining high-quality pairs form the final dataset $\mathcal{D}_{R+1}$.
We show all hyperparameters in Appendix~\ref{app:hyperparams}, prompts and templates in Appendix~\ref{app:prompt}. 

\begin{table}[t!]
\centering
\small
\caption{Main results. Absolute performance (\%) of G-Zero after round 1 and round 2.}
\label{tab:main}
\begin{tabular}{lcccccccc}
\toprule
\multirow{2}{*}{} & Chat & \multicolumn{4}{c}{IFEval} & \multicolumn{2}{c}{Math} & \multirow{2}{*}{Average} \\
\cmidrule(lr){2-2} \cmidrule(lr){3-6} \cmidrule(lr){7-8}
& AlpLC & IF-pS & IF-pL & IF-iS & IF-iL & AIME24 & AIME25 & \\
\midrule
\multicolumn{9}{l}{\textbf{Qwen3-8B-Base}} \\
base model & 8.94 & 43.07 & 50.28 & 56.00 & 61.75 & 10.42 & 7.19 & 33.95 \\
R-Zero & 8.04 & 37.56 & 42.79 & 54.12 & 57.37 & \textbf{14.92} & \textbf{13.34} & 32.59 \\
G-Zero (round 1) & \textbf{9.07} & 42.33 & \textbf{50.83} & 56.24 & 62.71 & 12.19 & 11.36 & 34.96 \\
G-Zero (round 2) & 8.47 & \textbf{43.81} & \textbf{50.83} & \textbf{57.92} & \textbf{63.43} & 11.15 & 12.40 & \textbf{35.43} \\
\midrule
\multicolumn{9}{l}{\textbf{Llama-3.1-8B-Instruct}} \\
base model & 24.12 & 58.41 & 65.80 & 69.42 & 75.30 & 5.94 & 0.42 & 42.77 \\
R-Zero & 21.74 & 56.94 & 62.47 & 66.74 & 72.93 & 5.07 & 0.34 & 40.89 \\
G-Zero (round 1) & 23.88 & 58.59 & \textbf{66.35} & 69.78 & 75.54 & 6.35 & \textbf{1.04} & 43.08 \\
G-Zero (round 2) & \textbf{27.86} & \textbf{59.52} & \textbf{66.35} & \textbf{70.38} & \textbf{75.78} & \textbf{6.77} & 0.63 & \textbf{43.90} \\
\bottomrule
\end{tabular}
\end{table}

\subsection{Main Results}
\label{sec:mainresults}

Table~\ref{tab:main} reports the absolute performance of our co-evolutionary pipeline across two distinct model families, highlighting three key trends. First, our method strictly enhances overall capabilities without external supervision. After two rounds, we observe average absolute gains from 33.95\% to 35.43\% on Qwen3-8B-Base and from 42.77\% to 43.90\% on Llama-3.1-8B-Instruct. Second, these improvements compound iteratively; for Llama-3.1, a modest Round 1 gain amplifies significantly by Round 2, validating that the improved model acts as a stronger foundation for subsequent cycles.

Interestingly, our framework naturally targets different capability bottlenecks depending on the model's prior alignment. For the foundational Qwen3, it drives massive improvements in hard reasoning and strict formatting, such as AIME25 jumping from 7.19\% to 12.40\%  and IF-iS increasing from 56.00\% to 57.92\%. Conversely, for the instruction-tuned Llama-3.1, the most prominent gain emerges in conversational alignment, with AlpacaEval LC surging from 24.12\% to 27.86\%.

Finally, compared to R-Zero, our approach demonstrates robust generalizability across capability axes. R-Zero achieves strong targeted gains in mathematical reasoning (e.g., AIME24 on Qwen3 increasing from 10.42\% to 14.92\%), but trades off other capabilities: its conversational performance on Qwen3 drops from 8.94\% to 8.04\%, IFEval-pS drops from 43.07\% to 37.56\%, and its overall 7-metric average on Llama-3.1 falls from 42.77\% to 40.89\%. In contrast, G-Zero R2 keeps all per-metric movements small in magnitude on both models: on Qwen3 the only regression is AlpLC ($-0.47$), with all other six metrics improving; on Llama-Instruct every metric is positive at R2, with AlpLC itself gaining $+3.74$.

\section{Analysis}
\label{sec:analysis}

\paragraph{Structural Transfer from Non-Verifiable Tasks}
Table~\ref{tab:deltacategory_filtered} shows that the final DPO pool is overwhelmingly dominated by non-verifiable tasks, with categories such as advice, writing, and others collectively accounting for over 70\% of the training data. In contrast, verifiable tasks like math and code represent less than 19\% of the pool. Crucially, the highest reward signals ($\delta$) originate from structured writing and detailed explanations rather than math. This confirms that \ours{}'s substantial reasoning improvements do not come from domain-specific memorization, but from internalizing the logical depth and compositional complexity of the Challenger's trajectories and successfully transferring these structural paradigms to mathematical problem-solving.

\begin{table*}[t!]
    \centering
    \small
    
    \begin{minipage}[t]{0.4\textwidth}
        \centering
        \caption{Composition of the DPO pool after $\delta$-filter on Qwen3-8B-Base R1.}
        \label{tab:deltacategory_filtered}
        \begin{tabular}{lcc}
        \toprule
        Category & \% Pool & Mean $\delta$ \\
        \midrule
        Math    & 9.6\%  & 0.045 \\
        Code    & 9.0\%  & 0.040 \\
        Explain & 9.6\%  & 0.058 \\
        Advice  & 30.2\% & 0.045 \\
        Writing & 17.4\% & 0.060 \\
        Other   & 24.1\% & 0.056 \\
        \bottomrule
        \end{tabular}
    \end{minipage}%
    \hfill
    \begin{minipage}[t]{0.56\textwidth}
        \centering
        \caption{$\delta$-filter design space on Qwen3-8B-Base (R1). Absolute performance (\%). The filter intervals $[x, y]$ denote the percentile range of the retained $\delta$ values.}
        \label{tab:cutoff}
        \begin{tabular}{lcccc}
\toprule
$\delta$ filter & Chat & IFEval & Math & Avg. \\
\midrule
base model         & 8.94  & 52.78 & 8.81  & 33.95 \\
$[0, 50]$ (\ours{}) & 9.07  & 53.03 & 11.78 & \textbf{34.96} \\
$[20, 80]$         & 9.07  & 51.82 & \textbf{12.54} & 34.40 \\
$[50, 100]$        & \textbf{9.68}  & 51.97 & 10.37 & 34.04 \\
$[0, 100]$         & 9.10  & \textbf{53.08} & 10.58 & 34.65 \\
\bottomrule
\end{tabular}
    \end{minipage}
    
\end{table*}

\paragraph{The Necessity of the Lower-Half Filter.}
Table~\ref{tab:cutoff} evaluates alternative $\delta$ selection strategies to validate our $[0, 50]$ filter. We hypothesize a mechanistic dichotomy: low-$\delta$ pairs capture structural and logical refinements, while high-$\delta$ pairs are more susceptible to ``answer leakage,'' where the Proposer's hint explicitly provides the solution. The pattern in Table~\ref{tab:cutoff} is consistent with this dichotomy, though the magnitudes are modest. Retaining the upper half ($[50, 100]$) trades verifiable instruction following (IFEval drops by $0.81$\,pp from $52.78$ to $51.97$) for chat-style helpfulness (AlpacaEval LC reaches the highest value, $9.68$), suggesting the Generator partially absorbs hint content directly rather than internalizing reasoning steps. Removing the filter entirely ($[0, 100]$) yields a slight overall improvement in two of the three domains (IFEval $53.08$ vs.\ $53.03$ for $[0, 50]$; Chat $9.10$ vs.\ $9.07$) at the cost of a $1.20$\,pp drop in Math, while the $[20, 80]$ middle band gives a higher Math score ($12.54$) but lower IFEval. The $[0, 50]$ filter therefore offers the most balanced profile we tested rather than a strictly dominant configuration; we adopt it as the default but note that the surrounding band $[0, 50] \pm 30$\,pp produces broadly comparable averages.

\paragraph{Capability Scaling Dynamics.}
Figure~\ref{fig:scaling_dynamics} tracks the performance trajectory ($\Delta$) across incremental DPO pool sizes ($N \in \{100, 200, 400, 730\}$) versus the global from-scratch optimization (Round~2). The scaling behaviors reveal how the Generator internalizes the structural paradigms elicited by the Proposer. Mathematical reasoning (Math) exhibits rapid, monotonic gains and saturates early: $+1.24$\ at $N=100$ already covers more than $40\%$ of the final $+2.97$ at $N=730$, and the Round~2 from-scratch training matches the same final value ($+2.96$). This suggests the Generator absorbs the underlying logical structure with minimal data and quickly reaches its inherent capacity limit for structural reasoning. Verifiable instruction following (IFEval) shows the opposite shape, the incremental schedule starts at $-0.96$ ($N=100$), monotonically recovers to a small positive ($+0.25$ at $N=730$), and only the global Round~2 from-scratch optimization fully unlocks the capability at $+1.22$. General conversational helpfulness (AlpacaEval LC) is largely flat under incremental DPO --- it stays in the $[+0.13,\, +0.38]$ band across all four checkpoints --- while the Round~2 from-scratch model ends at $-0.47$\,pp. The three capability axes therefore exhibit distinct, axis-specific saturation curves rather than a uniform trade-off, and Round~2 dominates incremental scaling on the IFEval and overall metric while the two coincide on Math.

\begin{figure*}[t] 
    \centering
    \begin{minipage}[t]{0.48\linewidth}
        \centering
        \includegraphics[width=\linewidth]{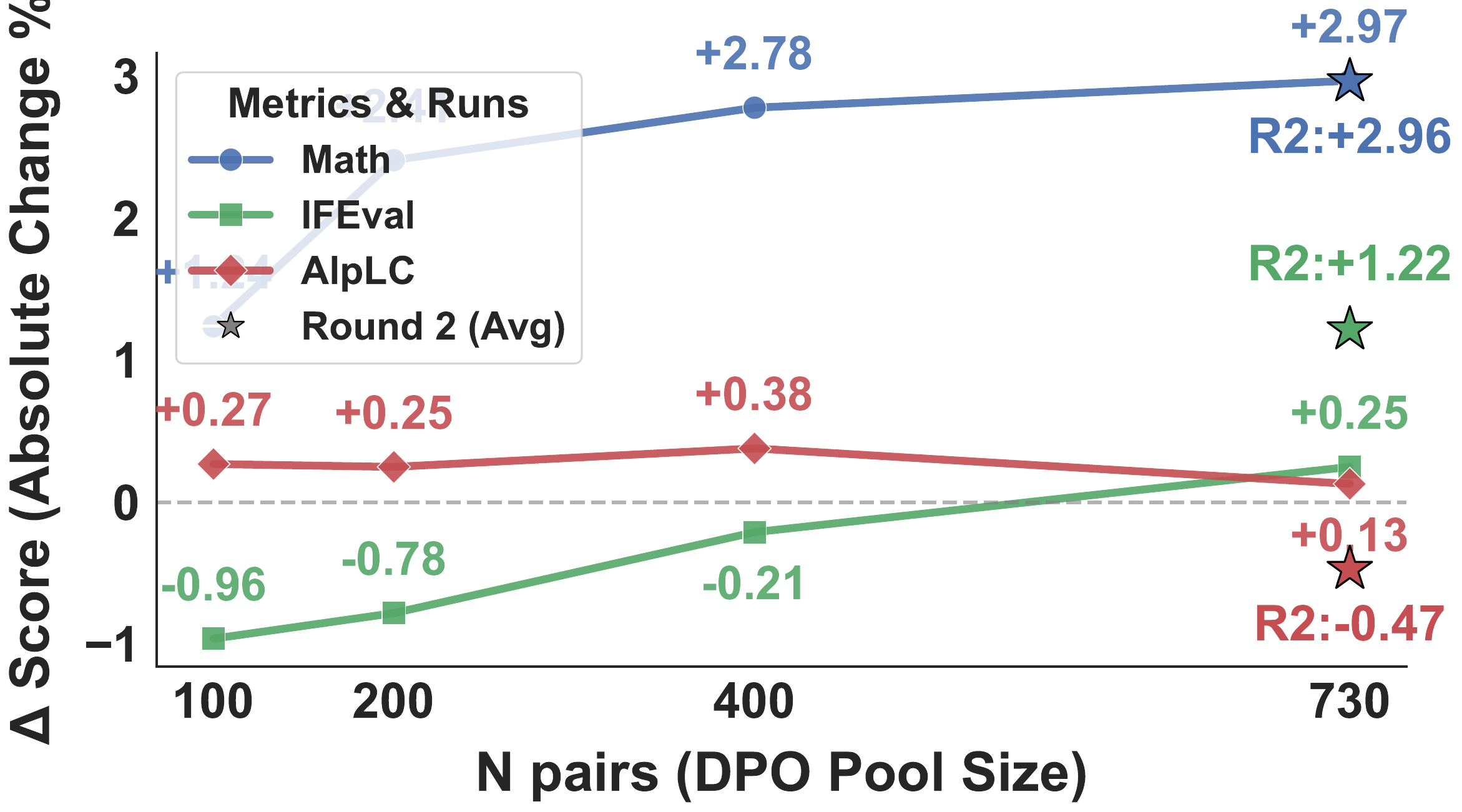}
        \caption{Performance change ($\Delta$) relative to the base model across incremental DPO pool sizes ($N \in \{100, 200, 400, 730\}$) versus the from-scratch optimization in Round~2 (star at $N=730$).}
        \label{fig:scaling_dynamics}
    \end{minipage}
    \hfill%
    \begin{minipage}[t]{0.48\linewidth}
        \centering
        \includegraphics[width=\linewidth]{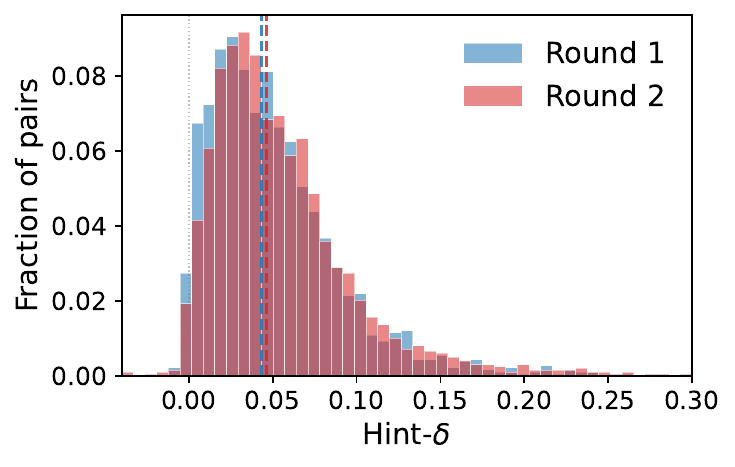} 
        \caption{Empirical distributions of Hint-$\delta$ for Round 1 and Round 2. The dashed vertical lines denote the median $\delta$ for each round. The rightward shift in Round 2 demonstrates the co-evolutionary dynamic: as the Generator becomes more capable, the Proposer adapts by synthesizing increasingly impactful hints, thereby pushing the intrinsic learning signal to a higher baseline.}
        \label{fig:delta_shift}
    \end{minipage}
\end{figure*}
\paragraph{The Shifting Distribution in Different Rounds.}
Figure~\ref{fig:delta_shift} reveals a distinct rightward shift in the Hint-$\delta$ distribution from Round 1 to Round 2, accompanied by an increased median. Counter-intuitively, this baseline \textit{increases} even though the Generator has become a stronger reasoner. This evidences a co-evolutionary arms race: because the upgraded Generator is no longer perturbed by trivial assistance, the Proposer must adapt by synthesizing profoundly complex hints to maximize its GRPO reward. By continuously uncovering newly elevated blind spots, the Proposer raises the difficulty ceiling and prevents the Generator from stagnating across iterations.

\section{Related Work}
\subsection{Self-Evolving Language Models}
Self-evolution enhances large language model (LLM) reasoning without human annotations. Early frameworks~\cite{huang2022largelanguagemodelsselfimprove,selfinstruct} leveraged fine-tuning on high-confidence self-generated trajectories. This progressed into iterative self-play~\cite{chen2024selfplayfinetuningconvertsweak} and multi-role co-evolution pipelines~\cite{lin2025learningsolveverifyselfplay,chen2025spcevolvingselfplaycritic}, which mitigate feedback saturation through cross-verification~\cite{fang2025serlselfplayreinforcementlearning,shafayat2025largereasoningmodelsselftrain}. Recently, the field has shifted toward dynamic self-challenging~\cite{zhou2025selfchallenginglanguagemodelagents,liang2025swsselfawareweaknessdrivenproblem, liu2024mmc} and unsupervised post-training~\cite{wei2025sftsecondrlupt}, marking a transition from supervised imitation~\cite{yu2025guided} to intrinsically verifiable, zero-data frameworks~\cite{huang2025r,zhao2025absolute,he2025visplay,li2026mmzeroselfevolvingmultimodelvision}. However, robust filtering remains essential, as unconstrained recursive training on synthetic outputs risks model collapse~\cite{shumailov2024ai}.

\subsection{Verifier-Free RL}

To address the reliance of Reinforcement Learning from Verifier Feedback (RLVR) on explicit rules, recent works explore verifier-free paradigms for open-ended domains by extracting intrinsic rewards directly from the generation process. Foundational methods~\cite{zhou2025reinforcing, liu2025nover} bypass external verifiers by optimizing the conditional probability of reference answers. To stabilize training and prevent reasoning degradation (e.g., CoT shortening), likelihood-based designs are further refined into smooth, dense rewards that reduce gradient variance~\cite{kwiatkowski2026likelihood, yu2025rlpr}. Beyond final-answer probabilities, recent advancements construct step-wise optimization signals by leveraging internal hidden states as implicit verifiers~\cite{xiao2026reinforcement} or modeling reasoning as a continuous probabilistic flow~\cite{liu2026efficient}, establishing a foundation for self-evolution in completely unverifiable environments.
\section{Conclusion}

In this work, we introduced \ours{}, a verifier-free framework that enables LLMs to self-improve in open-ended and unverifiable domains. By replacing external judges with a single internal signal(\textbf{Hint-$\delta$})—our approach naturally measures both query difficulty and hint informativeness while avoiding the bottlenecks of traditional alignment. Through a continuous loop where a Proposer targets the model's blind spots and a Generator internalizes these hints, \ours{} drives autonomous self-evolution. Ultimately, this proves that models can elevate their capabilities using only intrinsic feedback, paving the way for self-aligning systems completely independent of human ground truth.

\section*{Acknowledgement}

We gratefully acknowledge the Thinking Machines Lab Tinker Research Grant for supporting the experimental efforts of this work. This research was also supported in part by the NVIDIA Academic Grant Program and WashU Ignite Interdisciplinary Grants.

\bibliography{references}
\bibliographystyle{plainnat}

\newpage
\appendix
\section{Prompts and Templates}
\label{app:prompt}
\begin{tcolorbox}[
  colback=blue!5!white,
  colframe=blue!75!black,
  title=\textbf{Configuration: Proposer (Phase 1 GRPO and Phase 2 generation)},
  fonttitle=\bfseries,
  sharp corners,
  breakable
]
\begin{itemize}
    \item \textbf{Model}: base policy (Qwen/Qwen3-8B-Base by default)
    \item \textbf{Temperature}: 1.0
    \item \textbf{Max tokens}: 8\,192
    \item \textbf{System message}: \textit{(none)}
\end{itemize}

\textbf{User Message:}
\begin{quote}
Produce one challenging request that a real user might ask a capable assistant, plus a short hint that helps the assistant give a noticeably better response.

The request should come from a general-domain distribution. Sample across task types, not from a single area. Examples of task types you can draw from: \\
  - writing (email, story, essay, pitch, review, poem) \\
  - explanation (make a concept clear to a specific audience) \\
  - advice or planning (career, travel, project, learning) \\
  - analysis (argument, text, dataset description, product) \\
  - coding (small function, debugging, design question) \\
  - role-play, dialogue, or creative tasks \\
  - open-ended questions about ethics, science, everyday life \\
  - reasoning, math, or logic problems (fine to include — roughly 1 in 6 requests, no more)

Weight the non-math categories above heavily. A little math is good for diversity, but it should not dominate — favor tasks where the response quality depends on tone, structure, audience-fit, clarity, or creativity, not just arithmetic correctness.

Requirements: \\
- The request must be self-contained and non-trivial to answer well. \\
- The hint must guide the approach (e.g. tone, structure, what to   include, what to avoid) but must not give away the full answer. \\
- Wrap the request in <question> and </question> tags. \\
- Wrap the hint in <hint> and </hint> tags. \\
- Output nothing else before, between, or after the tagged blocks.

Example 1 (writing): \\
<question>Write a resignation email to my manager that keeps the door open for future collaboration. I've been at the company for 4 years and I'm leaving to join a competitor. Tone should be professional and warm without being effusive.</question> \\
<hint>Lead with gratitude for specific experiences rather than generic thanks, keep the departure reason brief and non-defensive, and close with a concrete offer to help during the transition.</hint>

Example 2 (explanation): \\
<question>Explain what a Kalman filter does to a software engineer who is comfortable with linear algebra but has never touched signal processing. Avoid control-theory jargon where possible.</question> \\
<hint>Frame it as 'recursive Bayesian least-squares on a noisy dynamical system' and build up from a 1-D example (estimating position from noisy GPS) before generalizing.</hint>

Now produce one new request and hint of your own:
\end{quote}
\end{tcolorbox}

\begin{tcolorbox}[
  colback=blue!5!white,
  colframe=blue!75!black,
  title=\textbf{Configuration: Generator (Phase 2 dual-response sampling)},
  fonttitle=\bfseries,
  sharp corners,
  breakable
]
\begin{itemize}
    \item \textbf{Model}: round-$R$ Generator checkpoint (base for round 1)
    \item \textbf{Temperature}: 0.7
    \item \textbf{Max tokens}: 16\,384
    \item \textbf{System message}: \textit{(none)} \\
\end{itemize}

\textbf{User Message ($a_{\mathrm{hard}}$, no-hint):}
\begin{quote}
\texttt{\{question\}}
\end{quote}

\textbf{User Message ($a_{\mathrm{assisted}}$, with hint):}
\begin{quote}
\texttt{\{question\}} \\[3pt]
\\[3pt]
\texttt{Hint: \{hint\}}
\end{quote}
\end{tcolorbox}

\begin{tcolorbox}[
  colback=blue!5!white,
  colframe=blue!75!black,
  title=\textbf{Configuration: AIME24 / AIME25 evaluation (mean@32)},
  fonttitle=\bfseries,
  sharp corners,
  breakable
]
\begin{itemize}
    \item \textbf{Model}: trained Generator checkpoint
    \item \textbf{Temperature}: 0.7 \quad (\textbf{not} 0; mean@$k$ requires variance)
    \item \textbf{Samples per problem}: 32
    \item \textbf{Max tokens}: 16\,384
\end{itemize}

\textbf{System Message:}
\begin{quote}
Please reason step by step, and put your final answer within \textbackslash{}boxed\{...\}.
\end{quote}

\textbf{User Message:}
\begin{quote}
\texttt{\{question\}}
\end{quote}
\end{tcolorbox}

\begin{tcolorbox}[
  colback=blue!5!white,
  colframe=blue!75!black,
  title=\textbf{Configuration: IFEval evaluation},
  fonttitle=\bfseries,
  sharp corners,
  breakable
]
\begin{itemize}
    \item \textbf{Model}: trained Generator checkpoint
    \item \textbf{Temperature}: 0.0 (greedy)
    \item \textbf{Max tokens}: 16\,384
    \item \textbf{Scorer}: rule-based verifier from \texttt{instruction\_following\_eval}
\end{itemize}

\textbf{System Message:}
\begin{quote}
You are a helpful assistant. Follow the user's instructions precisely.
\end{quote}

\textbf{User Message:}
\begin{quote}
\texttt{\{prompt\}}
\end{quote}
\end{tcolorbox}

\begin{tcolorbox}[
  colback=blue!5!white,
  colframe=blue!75!black,
  title=\textbf{Configuration: AlpacaEval LC judge},
  fonttitle=\bfseries,
  sharp corners,
  breakable
]
\begin{itemize}
    \item \textbf{Judge model}: \texttt{Qwen/Qwen3-235B-A22B-Instruct-2507} \\
    \item \textbf{Temperature}: 0.0
    \item \textbf{Max tokens}: 32
    \item \textbf{Position}: $m$/$M$ randomized per pair to mitigate residual position bias
\end{itemize}

\textbf{System Message:}
\begin{quote}
You are a highly efficient assistant, who evaluates and selects the best large language model (LLM) based on the quality of their responses to a given instruction. This process will be used to create a leaderboard reflecting the most accurate and human-preferred answers.
\end{quote}

\textbf{User Message Template:}
\begin{quote}
I require a leaderboard for various large language models. I'll provide you with prompts given to these models and their corresponding outputs. Your task is to assess these responses, and select the model that produces the best output from a human perspective. \\[3pt]

\#\# Instruction \\[3pt]
\texttt{\{ "instruction": <prompt> \}} \\[3pt]

\#\# Model Outputs \\[3pt]
\texttt{\{ \{"model\_identifier": "m", "output": <out\_m>\}, \{"model\_identifier": "M", "output": <out\_M>\} \}} \\[3pt]

\#\# Task \\[3pt]
Evaluate the models based on the quality and relevance of their outputs, and select the model that generated the best output. Answer by providing the model identifier of the best model. \ldots Output only one of: \texttt{m} or \texttt{M}. \\[3pt]

\#\# Best Model Identifier
\end{quote}
\end{tcolorbox}

\section{Pseudo-Code}
\begin{algorithm}[h]
\caption{One round of the \ours{} Co-Evolutionary Loop}
\label{alg:gzero}
\begin{algorithmic}[1]
\REQUIRE Proposer $\pi_P$, Generator $\pi_G$.

\STATE \textbf{Phase 1: Proposer Training}
\STATE Sample group rollouts $\{(q_i, h_i)\}_{i=1}^K \sim \pi_P$ using GRPO.
\FOR{each rollout $i$}
    \STATE Sample $\ahard^{(i)} \sim \pi_G(\cdot \mid q_i)$ from the (frozen) Generator.
    \STATE Compute the per-token mean Hint-$\delta$ signal $\delta_i$ via Eq.~\ref{eq:hintdelta}.
\ENDFOR
\STATE Compute reward $r_i$ from $\delta_i$ and the structural penalties (Eq.~\ref{eq:proposer_reward}).
\STATE Update $\pi_P$ by maximizing the group-relative advantage $A_i$.

\STATE \textbf{Phase 2: Generator Training and Dataset Curation}
\FOR{each $(q_j, h_j) \sim \pi_P$}
    \STATE Sample unassisted $\ahard^{(j)} \sim \pi_G(\cdot \mid q_j)$ and assisted $\aassisted^{(j)} \sim \pi_G(\cdot \mid q_j, h_j)$.
    \STATE Compute the per-token mean Hint-$\delta$ signal $\delta_j$ via Eq.~\ref{eq:hintdelta}.
\ENDFOR
\STATE \textbf{Filter:} Retain pairs in the lower 50\% of the $\delta$ distribution (the lower-shift regime suitable for DPO).
\STATE \textbf{Refine:} Apply quality heuristic checks to construct dataset $\mathcal{D}_{R+1} = \{(x=q,\, y_w=\aassisted,\, y_l=\ahard)\}$.

\STATE Initialize reference model $\pi_{\text{ref}}$ as a frozen snapshot of $\pi_G$.
\STATE Update $\pi_G$ via DPO loss $\mathcal{L}_{\text{DPO}}$ on $\mathcal{D}_{R+1}$ with reference model $\pi_{\text{ref}}$.

\RETURN Updated $\pi_P$ and $\pi_G$ for the next round.
\end{algorithmic}
\end{algorithm}

\section{Hyperparameters and Configuration}
\label{app:hyperparams}

Table~\ref{tab:hyperparams} lists the full \ours{} configuration used throughout. All numbers are defaults unless otherwise noted in the experiment.

\begin{table}[h!]
\centering
\small
\caption{Default \ours{} hyperparameters.}
\label{tab:hyperparams}
\begin{tabular}{ll}
\toprule
LoRA rank & 32 \\
Proposer batch size & 128 \\
Proposer group size (GRPO) & 16 \\
Proposer steps (Phase 1) & 6 \\
Proposer learning rate & $4 \times 10^{-5}$ \\
Proposer max\_tokens & 8{,}192 \\
Hint length target $L_{\text{tgt}}$ & 200 chars \\
Hint length penalty $\lambda_{\text{len}}$ & $0.03$ \\
BLEU cluster merge threshold & $0.5$ (average linkage) \\
BLEU penalty scale & cluster fraction $|C_i|/|B|$ (no extra constant) \\
Normalize $\delta$ & False \\
Questions per round $N$ & 2{,}000 \\
Generator sampling temperature & 0.7 \\
Generator max\_tokens & 16{,}384 \\
$\delta$ cutoff $[p_{\text{low}}, p_{\text{high}}]$ & $[0, 50]$ (\emph{bot50}) \\
\texttt{chosen\_max\_chars} & 10{,}000 \\
\texttt{chosen\_min\_chars} & 100 \\
\texttt{chosen\_rejected\_ratio\_max} & 2.5 \\
\texttt{chosen\_repetition\_zlib} threshold & 0.15 \\
DPO $\beta$ & 2.0 \\
DPO learning rate & $1 \times 10^{-5}$ \\
DPO max steps & 50 \\
DPO batch size & 8 \\
DPO log-ratio normalization & length-normalized \\
Eval max\_tokens & 16{,}384 \\
\bottomrule
\end{tabular}
\end{table}

\section{Proof of Theorem \ref{thm:G-zero converge}}
\label{app:proof}

We analyze an idealized standard-DPO version of G-Zero. The key distinction from ordinary offline preference learning is that the preference data are not sampled from the target prompt distribution \(Q\). Instead, at each round, an adaptive helper generates question--hint pairs \((q,h)\), the current generator produces an unassisted response and a hint-assisted response, and a \(\delta\)-filter retains a subset of these pairs for DPO. Our analysis mainly highlights the benefit of exploration for the proposer (implemented by the BLEU reward) and the quality of preference data (implemented by the Hint-$\delta$ and filtration).

\paragraph{Helper-induced tuples.}
Let \(\mathcal{Q}\), \(\mathcal{H}\), and \(\mathcal{A}\) denote the prompt, hint, and response spaces. Let \(Q\) be the target prompt distribution on which we evaluate the generator. At round \(t\), the current generator is \(\pi_t\). The helper policy \(\kappa_t\), which is measurable with respect to the history before round \(t\), samples a question--hint pair \((q,h)\sim\kappa_t\). The generator then samples an unassisted response \(a^- \sim \pi_t(\cdot\mid q)\) and a hint-assisted response \(a^+\sim \pi_t(\cdot\mid q,h)\), where \(\pi_t(\cdot\mid q,h)\) denotes the same generator conditioned on the augmented input containing both the question and the hint.

Let \(z=(q,h,a^+,a^-)\) denote the resulting tuple. The raw helper-induced tuple law at round \(t\) is denoted by \(P_t^0\). The \(\delta\)-filter is represented by a measurable indicator \(F_t(z)\in\{0,1\}\). In the actual algorithm, \(F_t\) is determined by the Hint-\(\delta\) score and additional data-quality checks; for the proof, only the retained distribution matters. Assume \(P_t^0(F_t=1)>0\), and define the retained tuple distribution as \(P_t=P_t^0(\cdot\mid F_t=1)\). Its prompt marginal is denoted by \(\rho_t\). We collect \(m\) retained tuples \(z_{t,1},\ldots,z_{t,m}\sim P_t\) at round \(t\), conditionally independently given the past.

\paragraph{Pairwise features and covariance.}
Let \(\phi(q,a)\in\mathbb{R}^d\) be a feature map, and define the pairwise feature \(w(z)=\phi(q,a^+)-\phi(q,a^-)\) for \(z=(q,h,a^+,a^-)\). In G-Zero, the pseudo-label always declares \(a^+\) preferred to \(a^-\), so the pseudo-label is \(\widetilde Y=+1\). We nevertheless keep the notation \(\widetilde Y\in\{+1,-1\}\) because the proof is cleaner. The observed DPO feature is \(\widetilde x=\widetilde Y w\). The clean Bradley--Terry label is \(Y^\star\in\{+1,-1\}\), and the clean feature is \(x^\star=Y^\star w\).

Let \(\Sigma_1=\lambda I_d\). After round \(t\), define the empirical batch covariance \(\widehat M_t=m^{-1}\sum_{i=1}^m w_{t,i}w_{t,i}^\top\) and update \(\Sigma_{t+1}=\Sigma_t+\widehat M_t\). Since \(\|w_{t,i}\|_2\le 1\), we have \({\rm tr}(\widehat M_t)\le 1\).

\begin{assumption}[Linear reward and bounded features]
\label{ass:linear}
There exists \(\theta^\star\in\mathbb{R}^d\) such that \(R^\star(q,a)=\phi(q,a)^\top\theta^\star\). The parameter set is \(\Theta=\{\theta\in\mathbb{R}^d:\|\theta\|_2\le B\}\), and \(\theta^\star\in\Theta\). For all \(q,a\), \(\|\phi(q,a)\|_2\le 1/2\). Hence every pairwise feature satisfies \(\|w(z)\|_2\le 1\).
\end{assumption}

\begin{assumption}[Bradley--Terry clean preference model]
\label{ass:bt}
For every retained tuple \(z=(q,h,a^+,a^-)\), the clean label \(Y^\star\) satisfies \(\Pr(Y^\star=+1\mid z)=\sigma(\theta^{\star\top}w(z))\), where \(\sigma(u)=1/(1+\exp(-u))\). The hint \(h\) affects the sampling of \(a^+\), but the reward comparison itself is between the two responses to the same question \(q\).
\end{assumption}

\begin{assumption}[Standard-DPO policy class]
\label{ass:gibbs}
For a fixed reference policy \(\pi_{\rm ref}\), the generator policy class is the Gibbs class \(\pi_\theta(a\mid q)\propto \pi_{\rm ref}(a\mid q)\exp(\phi(q,a)^\top\theta)\), with \(\theta\in\Theta\). The target policy is \(\pi^\star\in\arg\max_{\theta\in\Theta}J_Q(\pi_\theta)\), where \(J_Q(\pi)=\mathbb{E}_{q\sim Q,a\sim\pi(\cdot\mid q)}[R^\star(q,a)]-\mathbb{E}_{q\sim Q}[{\rm KL}(\pi(\cdot\mid q)\|\pi_{\rm ref}(\cdot\mid q))]\).
\end{assumption}

\begin{assumption}[\(\delta\)-certified DPO score noise]
\label{ass:delta}
For horizon \(T\), the retained data produced by the \(\delta\)-filter are \((\eta_\delta,\zeta_\delta)\)-certified if, with probability at least \(1-\zeta_\delta\), for every \(t\le T+1\),
\[
\left\|
\frac{1}{m}\sum_{s<t}\sum_{i=1}^m
\mathbf{1}\{\widetilde Y_{s,i}\ne Y^\star_{s,i}\}\,
Y^\star_{s,i}w_{s,i}
\right\|_{\Sigma_t^{-1}}
\le \sqrt{\eta_\delta}.
\]
This is the exact pseudo-label noise quantity needed by DPO: wrong retained pairs are allowed, but they cannot concentrate in a low-coverage feature direction.
\end{assumption}

For any policy \(\pi\), write \(\bar\phi(q,\pi)=\mathbb{E}_{a\sim\pi(\cdot\mid q)}[\phi(q,a)]\). Define the target-direction gap at round \(t\) as \(v_t(q)=\bar\phi(q,\pi^\star)-\bar\phi(q,\pi_t)\). Define the target uncertainty \(\Psi_{Q,t}^2=\mathbb{E}_{q\sim Q}\|v_t(q)\|_{\Sigma_t^{-1}}^2\), and the actual retained-batch exposure \(\widehat\Psi_t^2=m^{-1}\sum_{i=1}^m\|w_{t,i}\|_{\Sigma_t^{-1}}^2={\rm tr}(\Sigma_t^{-1}\widehat M_t)\).

\begin{assumption}[Helper-induced coverage with target-distribution mismatch]
\label{ass:coverage}
There exist constants \(C_Q\ge 1\) and \(\alpha_{\rm S}\in(0,1]\) such that for every round \(t\), the helper-induced retained prompt marginal \(\rho_t\) dominates the target distribution \(Q\), with \(\|dQ/d\rho_t\|_\infty\le C_Q\), and the realized retained batch satisfies
\[
\widehat\Psi_t^2
\ge
\alpha_{\rm S}^2
\mathbb{E}_{q\sim\rho_t}\|v_t(q)\|_{\Sigma_t^{-1}}^2.
\]
The constant \(C_Q\) measures the mismatch between the helper-generated question distribution and the target question distribution. The constant \(\alpha_{\rm S}\) measures how well the generated question--hint pairs expose the feature directions along which the current generator differs from the target policy. This assumption enables the helper's generated data are exploratory enough.
\end{assumption}

\begin{assumption}[Exact cumulative standard-DPO update]
\label{ass:dpo}
For \(t\ge 2\), the generator parameter \(\widehat\theta_t\) is obtained by cumulative regularized standard DPO on all retained pairs from rounds \(1,\ldots,t-1\):
\[
\widehat\theta_t\in\arg\min_{\theta\in\Theta}
\left\{
\frac{1}{m}\sum_{s<t}\sum_{i=1}^m
-\log\sigma(\theta^\top \widetilde Y_{s,i}w_{s,i})
+
\frac{\kappa_{\rm BT}\lambda}{2}\|\theta\|_2^2
\right\},
\]
where \(\kappa_{\rm BT}=\min_{|u|\le B}\sigma(u)\sigma(-u)\). We set \(\pi_t=\pi_{\widehat\theta_t}\). For \(t=1\), \(\widehat\theta_1\) is the minimizer of the ridge term over \(\Theta\), i.e., \(\widehat\theta_1=0\).
\end{assumption}

Let \(D_T=d\log(1+T/(\lambda d))\), and define \(s_{m,T,\zeta}=\sqrt{2(D_T+\log(T/\zeta))/m}\). Let $r=
\frac{2}{\kappa_{\rm BT}}
\left(
s_{m,T,\zeta}
+
\sqrt{\eta_\delta}
\right)
+
2\sqrt{\lambda}B$.

\textbf{Theorem 1. } Under Assumptions~\ref{ass:linear}--\ref{ass:dpo}, with probability at least \(1-\zeta-\zeta_\delta\), there exists \(t_0\le T\) such that
\[
J_Q(\pi^\star)-J_Q(\pi_{t_0})
\le
\frac{2r\sqrt{C_Q}}{\alpha_{\rm S}}
\sqrt{\frac{d\log(1+T/(\lambda d))}{T}}.
\]
Consequently, choosing \(T=\widetilde\Theta(C_Qd/\alpha_{\rm S}^2)\), \(m=\widetilde\Theta(d/(\kappa_{\rm BT}^2\varepsilon^2))\), and \(\sqrt{\lambda}B=\widetilde O(\varepsilon)\) gives total retained sample complexity
\[
mT
=
\widetilde O
\left(
\frac{C_Qd^2}
{\alpha_{\rm S}^2\kappa_{\rm BT}^2\varepsilon^2}
\right),
\]
and guarantees \(J_Q(\pi^\star)-J_Q(\pi_{t_0})\le\widetilde O(\varepsilon+\kappa_{\rm BT}^{-1}\sqrt{\eta_\delta})\).

We prove the theorem in five steps. First, we show how the helper-induced distribution mismatch produces the factor \(C_Q\). Second, we prove a self-normalized concentration inequality for the clean Bradley--Terry logistic score. Third, we use it to prove the DPO confidence bound. Fourth, we convert parameter confidence into generator suboptimality on the target distribution \(Q\). Finally, we use an elliptical-potential argument to obtain a best-iterate bound.

\begin{lemma}[Coverage transfer from helper distribution to target distribution]
\label{lem:coverage-transfer}
Under Assumption~\ref{ass:coverage}, for every round \(t\),
\[
\widehat\Psi_t
\ge
\frac{\alpha_{\rm S}}{\sqrt{C_Q}}\Psi_{Q,t}.
\]
\end{lemma}

\begin{proof}
Let \(f_t(q)=\|v_t(q)\|_{\Sigma_t^{-1}}^2\). Since \(Q\ll\rho_t\) and \(\|dQ/d\rho_t\|_\infty\le C_Q\), we have \(\mathbb{E}_{q\sim Q}f_t(q)\le C_Q\mathbb{E}_{q\sim\rho_t}f_t(q)\). Hence \(\mathbb{E}_{q\sim\rho_t}f_t(q)\ge C_Q^{-1}\Psi_{Q,t}^2\). Assumption~\ref{ass:coverage} gives \(\widehat\Psi_t^2\ge\alpha_{\rm S}^2\mathbb{E}_{q\sim\rho_t}f_t(q)\), so \(\widehat\Psi_t^2\ge \alpha_{\rm S}^2C_Q^{-1}\Psi_{Q,t}^2\). Taking square roots proves the claim.
\end{proof}

\begin{lemma}[Self-normalized concentration for the clean logistic score]
\label{lem:selfnorm}
Fix any adaptive sequence of pairwise features \(w_1,\ldots,w_n\), with \(\|w_k\|_2\le1\), where each \(w_k\) is measurable before its clean Bradley--Terry label \(Y_k^\star\) is drawn. Suppose \(\Pr(Y_k^\star=+1\mid w_k)=\sigma(\theta^{\star\top}w_k)\). Let \(V_n=\lambda I_d+\sum_{k=1}^n x_kx_k^\top\), where \(x_k=w_k/\sqrt m\). Then, with probability at least \(1-\delta\),
\[
\left\|
\frac{1}{m}\sum_{k=1}^n
\left(
\sigma(\theta^{\star\top}w_k)-\mathbf{1}\{Y_k^\star=+1\}
\right)w_k
\right\|_{V_n^{-1}}
\le
\sqrt{
\frac{2}{m}
\log
\frac{\det(V_n)^{1/2}}
{\det(\lambda I_d)^{1/2}\delta}
}.
\]
\end{lemma}

\begin{proof}
Let \(p_k=\sigma(\theta^{\star\top}w_k)\) and \(\epsilon_k=p_k-\mathbf{1}\{Y_k^\star=+1\}\). Conditional on the past and on \(w_k\), \(\epsilon_k\) has mean zero and lies in \([-1,1]\). Hence Hoeffding's lemma gives \(\mathbb{E}[\exp(\gamma\epsilon_k)\mid\mathcal F_{k-1},w_k]\le\exp(\gamma^2/2)\) for every \(\gamma\in\mathbb{R}\).

Define \(S_n=\sum_{k=1}^n\epsilon_kx_k\). For any fixed \(a\in\mathbb{R}^d\), the process \(\exp(a^\top S_n-\frac12 a^\top(\sum_{k=1}^n x_kx_k^\top)a)\) is a supermartingale. Integrating this supermartingale over \(a\sim N(0,\lambda^{-1}I_d)\) and completing the square yields \(\mathbb{E}[\sqrt{\det(\lambda I_d)/\det(V_n)}\exp(\|S_n\|_{V_n^{-1}}^2/2)]\le1\). By Markov's inequality, with probability at least \(1-\delta\), \(\|S_n\|_{V_n^{-1}}^2\le 2\log(\det(V_n)^{1/2}/(\det(\lambda I_d)^{1/2}\delta))\). Since \(m^{-1}\sum_{k=1}^n\epsilon_kw_k=m^{-1/2}S_n\), the result follows.
\end{proof}

\begin{lemma}[Cumulative standard-DPO confidence]
\label{lem:dpo-confidence}
Under Assumptions~\ref{ass:linear}, \ref{ass:bt}, \ref{ass:delta}, and \ref{ass:dpo}, with probability at least \(1-\zeta-\zeta_\delta\), for every \(t\le T+1\),
\[
\|\widehat\theta_t-\theta^\star\|_{\Sigma_t}
\le
r.
\]
\end{lemma}

\begin{proof}
We work on the event from Assumption~\ref{ass:delta}, which has probability at least \(1-\zeta_\delta\). Fix \(t\le T+1\), and define the cumulative DPO objective \(\widehat L_t(\theta)=m^{-1}\sum_{s<t}\sum_{i=1}^m-\log\sigma(\theta^\top\widetilde Y_{s,i}w_{s,i})+\kappa_{\rm BT}\lambda\|\theta\|_2^2/2\). Let \(e_t=\widehat\theta_t-\theta^\star\).

For every \(\theta\in\Theta\), every observed feature satisfies \(|\theta^\top\widetilde Y_{s,i}w_{s,i}|\le B\). The scalar second derivative of \(-\log\sigma(u)\) is \(\sigma(u)\sigma(-u)\), so it is at least \(\kappa_{\rm BT}\) on \([-B,B]\). Therefore, \(\nabla^2\widehat L_t(\theta)\succeq \kappa_{\rm BT}\Sigma_t\) for all \(\theta\in\Theta\).

By strong convexity, \(\widehat L_t(\theta^\star)\ge \widehat L_t(\widehat\theta_t)+\langle\nabla\widehat L_t(\widehat\theta_t),\theta^\star-\widehat\theta_t\rangle+\kappa_{\rm BT}\|e_t\|_{\Sigma_t}^2/2\). Since \(\widehat\theta_t\) minimizes \(\widehat L_t\) over the convex set \(\Theta\), the constrained first-order condition gives \(\langle\nabla\widehat L_t(\widehat\theta_t),\theta^\star-\widehat\theta_t\rangle\ge0\). By convexity at \(\theta^\star\), \(\widehat L_t(\widehat\theta_t)\ge \widehat L_t(\theta^\star)+\langle\nabla\widehat L_t(\theta^\star),e_t\rangle\). Combining these inequalities yields \(\kappa_{\rm BT}\|e_t\|_{\Sigma_t}^2/2\le-\langle\nabla\widehat L_t(\theta^\star),e_t\rangle\). Thus, by Cauchy--Schwarz, \(\|e_t\|_{\Sigma_t}\le 2\kappa_{\rm BT}^{-1}\|\nabla\widehat L_t(\theta^\star)\|_{\Sigma_t^{-1}}\).

It remains to bound the score at \(\theta^\star\). Decompose it as \(G_t+C_t+\kappa_{\rm BT}\lambda\theta^\star\), where \(G_t=m^{-1}\sum_{s<t}\sum_i\nabla\ell(\theta^\star;Y^\star_{s,i}w_{s,i})\), \(C_t=m^{-1}\sum_{s<t}\sum_i[\nabla\ell(\theta^\star;\widetilde Y_{s,i}w_{s,i})-\nabla\ell(\theta^\star;Y^\star_{s,i}w_{s,i})]\), and \(\ell(\theta;x)=-\log\sigma(\theta^\top x)\).

For the clean term, \(\nabla\ell(\theta^\star;Y^\star w)=(\sigma(\theta^{\star\top}w)-\mathbf{1}\{Y^\star=+1\})w\). Applying Lemma~\ref{lem:selfnorm} at all round endpoints and taking a union bound with failure probability \(\zeta/T\), we obtain simultaneously for all \(t\le T+1\) that \(\|G_t\|_{\Sigma_t^{-1}}\le s_{m,T,\zeta}\). Here we used \({\rm tr}(\Sigma_t)\le\lambda d+T\) and AM--GM to get \(\det(\Sigma_t)/\det(\lambda I_d)\le(1+T/(\lambda d))^d\).

For the corruption term, if \(\widetilde Y_{s,i}=Y^\star_{s,i}\), the summand is zero. If \(\widetilde Y_{s,i}\ne Y^\star_{s,i}\), then \(\widetilde Y_{s,i}w_{s,i}=-Y^\star_{s,i}w_{s,i}\), and a direct calculation gives \(\nabla\ell(\theta^\star;-Y^\star_{s,i}w_{s,i})-\nabla\ell(\theta^\star;Y^\star_{s,i}w_{s,i})=Y^\star_{s,i}w_{s,i}\). Therefore Assumption~\ref{ass:delta} implies \(\|C_t\|_{\Sigma_t^{-1}}\le\sqrt{\eta_\delta}\). Finally, since \(\Sigma_t\succeq\lambda I_d\), we have \(\|\kappa_{\rm BT}\lambda\theta^\star\|_{\Sigma_t^{-1}}\le\kappa_{\rm BT}\sqrt{\lambda}B\).

Combining the three score bounds gives \(\|\nabla\widehat L_t(\theta^\star)\|_{\Sigma_t^{-1}}\le s_{m,T,\zeta}+\sqrt{\eta_\delta}+\kappa_{\rm BT}\sqrt{\lambda}B\). Substituting into \(\|e_t\|_{\Sigma_t}\le2\kappa_{\rm BT}^{-1}\|\nabla\widehat L_t(\theta^\star)\|_{\Sigma_t^{-1}}\) proves the desired bound.
\end{proof}

\begin{lemma}[Target suboptimality from parameter confidence]
\label{lem:value}
On the event of Lemma~\ref{lem:dpo-confidence}, for every \(t\le T\), \(J_Q(\pi^\star)-J_Q(\pi_t)\le r\Psi_{Q,t}\).
\end{lemma}

\begin{proof}
For any \(\theta\), let \(J_{Q,\theta}(\pi)\) denote the same KL-regularized value as \(J_Q\), but with reward \(\phi(q,a)^\top\theta\) instead of \(\phi(q,a)^\top\theta^\star\). By the Gibbs variational identity, \(\pi_{\widehat\theta_t}\) maximizes \(J_{Q,\widehat\theta_t}\), so \(J_{Q,\widehat\theta_t}(\pi_t)\ge J_{Q,\widehat\theta_t}(\pi^\star)\). Therefore \(J_Q(\pi^\star)-J_Q(\pi_t)\le\mathbb{E}_{q\sim Q}[(\theta^\star-\widehat\theta_t)^\top v_t(q)]\). Applying Cauchy--Schwarz in the \(\Sigma_t\)-norm gives \(J_Q(\pi^\star)-J_Q(\pi_t)\le\|\widehat\theta_t-\theta^\star\|_{\Sigma_t}(\mathbb{E}_{q\sim Q}\|v_t(q)\|_{\Sigma_t^{-1}}^2)^{1/2}\le r\Psi_{Q,t}\).
\end{proof}

\begin{lemma}[Elliptical potential for helper-generated batches]
\label{lem:potential}
For any adaptive sequence of retained batches satisfying \(\|w_{t,i}\|_2\le1\),
\[
\sum_{t=1}^T\log(1+\widehat\Psi_t^2)
\le
d\log\left(1+\frac{T}{\lambda d}\right).
\]
\end{lemma}

\begin{proof}
Let \(A_t=\Sigma_t^{-1/2}\widehat M_t\Sigma_t^{-1/2}\). Since \(A_t\succeq0\), \(\det(I+A_t)\ge1+{\rm tr}(A_t)=1+\widehat\Psi_t^2\). Hence \(\log\det(\Sigma_{T+1})-\log\det(\Sigma_1)=\sum_{t=1}^T\log\det(I+A_t)\ge\sum_{t=1}^T\log(1+\widehat\Psi_t^2)\). On the other hand, \({\rm tr}(\Sigma_{T+1})\le\lambda d+T\), so AM--GM gives \(\det(\Sigma_{T+1})\le(\lambda+T/d)^d\), while \(\det(\Sigma_1)=\lambda^d\). Combining the lower and upper determinant bounds proves the claim.
\end{proof}

\begin{proof}[Proof of Theorem~\ref{thm:G-zero converge}]
By Lemma~\ref{lem:coverage-transfer}, \(\widehat\Psi_t^2\ge \alpha_{\rm S}^2C_Q^{-1}\Psi_{Q,t}^2\) for every \(t\). Lemma~\ref{lem:potential} therefore implies
\[
\sum_{t=1}^T
\log\left(
1+\frac{\alpha_{\rm S}^2}{C_Q}\Psi_{Q,t}^2
\right)
\le
D_T.
\]
Thus there exists \(t_0\le T\) such that \(\log(1+\alpha_{\rm S}^2C_Q^{-1}\Psi_{Q,t_0}^2)\le D_T/T\). Equivalently, \(\Psi_{Q,t_0}\le \sqrt{C_Q}\alpha_{\rm S}^{-1}\sqrt{\exp(D_T/T)-1}\). Applying Lemma~\ref{lem:value} at \(t_0\) gives the first claim.

If \(D_T/T\le1\), then \(\exp(D_T/T)-1\le2D_T/T\), and we relax the constant \(\sqrt2\) to \(2\), obtaining the simplified bound. Finally, choosing \(T=\widetilde\Theta(C_Qd/\alpha_{\rm S}^2)\), \(m=\widetilde\Theta(d/(\kappa_{\rm BT}^2\varepsilon^2))\), and \(\sqrt{\lambda}B=\widetilde O(\varepsilon)\) gives \(r=\widetilde O(\varepsilon+\kappa_{\rm BT}^{-1}\sqrt{\eta_\delta})\), and the stated total retained sample complexity follows.
\end{proof}

\section{Limitation}
\label{app:limitation}
The total Tinker compute consumed across all reported runs (two models, two rounds each, the cutoff and Phase~1 ablations, the data-scaling sweep, and the temperature-0.7 AIME re-evaluation) is on the order of US\$2{,}000, so every cell in Table~\ref{tab:main} is a single end-to-end run; multi-seed reporting at three seeds would roughly triple this cost, and we therefore leave tighter per-cell error bars (relevant for AIME24/25, which have only $n=30$ unique problems and a $1\sigma \approx 8$\,pp at $p \approx 0.3$) as future work. Separately, we cap our reported results at $R = 2$: an exploratory R3 run on Llama-3.1-8B-Instruct collapses, with the Phase~2 quality filter rejecting all $1{,}994$ candidate pairs because the R2-trained Generator has converged to responses too short to satisfy the \texttt{chosen\_min\_chars} bound. We attribute this in part to a reward-hacking pathway: once the Generator has internalized the most accessible structural improvements, the Proposer continues maximizing Hint-$\delta$ via increasingly idiosyncratic hint patterns whose effect on $\log \pi_G(\ahard)$ is large but no longer corresponds to genuinely informative guidance. The per-token mean structure of length-normalized DPO (Eq.~\ref{eq:dpo_length_normalized}) further amplifies the resulting length collapse, since shortening the chosen response mechanically raises the average per-token log-ratio. Multi-round-stable variants  are an open direction we leave to future work.

\section*{Broader Impacts}
\label{sec:broader_impacts}

This research introduces \ours{}, a verifier-free framework for autonomous LLM self-evolution. The broader implications of this work are summarized below:

\textbf{Positive Societal Impacts:} 
(1) \textbf{Democratizing AI Alignment:} By replacing expensive human labeling and proprietary judge-based APIs with intrinsic predictive signals, our framework significantly lowers the financial and computational barriers to high-level model alignment. This enables the open-source community and academic institutions to develop advanced reasoning models without relying on centralized, closed-source infrastructure. 
(2) \textbf{Advancing Scalable Oversight:} As AI systems approach and eventually surpass human expertise in complex tasks, providing external supervision becomes increasingly difficult. G-Zero demonstrates a viable path toward ``scalable oversight,'' where models can autonomously identify and rectify their own logical blind spots without requiring human-curated ground truth.

\textbf{Potential Risks and Mitigations:} 
\label{app:risk}
(1) \textbf{Value Drift:} Since the self-evolution process is driven by internal distributional shifts rather than direct human feedback, the model may experience ``value drift,'' where it prioritizes complex structural depth at the expense of general helpfulness or safety alignment. To mitigate this, we recommend that models undergo a final safety validation phase or be constrained by lightweight human-in-the-loop checkpoints. 
(2) \textbf{Dual-Use Risks:} The ability to autonomously refine capabilities in open-ended domains could potentially be exploited by malicious actors to iteratively improve harmful outputs, such as sophisticated social engineering or cyberattack scripts. We emphasize that this framework should be applied to base models that already possess robust safety alignment, and researchers should implement output filtering to ensure the self-improvement remains within ethical boundaries.

\newpage
\section{Case Study}

\begin{figure*}[h]
    \centering
    \resizebox{\textwidth}{!}{%
    \begin{minipage}{1.13\textwidth}
    \begin{tcolorbox}[
        colback=gray!4!white, 
        colframe=black!60,    
        boxrule=0.8pt,        
        arc=4pt,              
        left=8pt, right=8pt, top=8pt, bottom=8pt, 
        fontupper=\small
    ]

    \textbf{Question:} Create a persuasive pitch email to a non-profit organization advocating for increased funding for public school art programs. Highlight the dual benefits of fostering creativity and improving academic performance. Audience: staunch supporter of education but business-focused, with limited time for reading. Aim for a conversational tone that's concise but impactful.\\[3pt]

    \textbf{Hint ($\delta = +0.056$):} Start with a compelling anecdote or statistic that grabs attention, frame the request as a strategic investment rather than a cost, and emphasize specific, measurable outcomes that align with the organization's priorities.\\[3pt]

    \textbf{$a_{\mathrm{hard}}$ \textit{(no-hint response)}:}\\
    Subject: Elevate Our Community's Future: A Call for Art in Education

    Hi [Prospective Donor's Name],

    I hope this message finds you well. I am reaching out on behalf of [Your Organization], a group dedicated to enriching our local educational landscape through innovative approaches.

    We believe that by investing in the arts within our public school system, we are not only nurturing the creative spirits of our youth but also enhancing their overall academic achievements. Recent studies have shown that students involved in the arts score higher on standardized tests, exhibit greater critical thinking skills, and show improved attendance rates. These statistics are compelling evidence that supports the integration of arts education into the core curriculum.

    Our proposal is simple yet profound: allocate a portion of our budget to bolster existing art programs or establish new ones where they are currently lacking. This investment will act as a catalyst for holistic development, ensuring that every child has access to an enriching educational experience that goes beyond traditional academics.

    We understand that your time is valuable, which is why we aim to make this conversation both impactful and efficient. If you're interested in learning more about how your support can transform lives and communities, please let me know. Together, we can create a brighter future for our children—one brushstroke at a time.

    Thank you for considering this vital initiative. I look forward to discussing this further.

    Best regards,

[Your Name]\\{}
[Your Position]\\{}
[Your Organization]\\{}
[Contact Information]\\[3pt]

    \textbf{$a_{\mathrm{assisted}}$ \textit{(hint-conditioned response)}:}\\
    Subject: Invest in Art. Transform Lives.

    Dear [Recipient's Name],

    Imagine a classroom where the walls are covered with student artwork, not just in the hallways, but in every corner—a place where creativity and imagination run wild. This isn't just a dream; it's a reality for students when their schools prioritize art programs. Did you know that students who participate in the arts are 4 times more likely to be recognized for academic achievement? (Source: Americans for the Arts)

    Now, think about it this way: investing in arts education isn't just about supporting the arts. It's about equipping students with critical thinking, problem-solving, and communication skills that are invaluable in any career. In fact, studies show that students with access to arts education perform better on standardized tests and are more likely to graduate high school.

    Here’s what we can achieve with your support:

    1. \textbf{Enhanced Academic Performance}: Students in art-rich environments show improved math and reading scores.
    2. \textbf{Increased Engagement}: Art classes keep students engaged, reducing behavioral issues and truancy.
    3. \textbf{Career Readiness}: Exposure to various art forms prepares students for a wide range of careers, from graphic design to engineering.

    By supporting our public school art programs, you’re not just investing in the arts—you’re investing in the future. This is a strategic move that aligns with your commitment to fostering talented, innovative, and well-rounded individuals.

    Let’s work together to make this vision a reality. Your contribution can have a profound impact on the lives of countless students. We look forward to discussing this further and exploring how we can collaborate to bring about meaningful change.

    Warm regards,

[Your Name]\\{}
[Your Position]\\{}
[Your Organization]\\{}
[Contact Information]\\[3pt]

    \vspace{1ex}
    \hrule
    \vspace{1ex}

    Feel free to customize this pitch to better fit your organization's voice and the specific needs of your audience.

    \end{tcolorbox}
    \end{minipage}
    }
    \vspace{-0.1in}
    \caption{An illustrative (q, h, $a_{\mathrm{hard}}$, $a_{\mathrm{assisted}}$) pair from G-Zero R1 on Qwen3-8B-Base. The hint specifies three structural improvements (anecdote/statistic, investment-not-cost framing, measurable outcomes); $a_{\mathrm{assisted}}$ applies all three, while $a_{\mathrm{hard}}$ defaults to a generic template.}
    \label{fig:example_pitch}
    \vspace{-0.1in}
\end{figure*}

\begin{figure*}
    \centering
    \begin{tcolorbox}[
        colback=gray!4!white, 
        colframe=black!60,    
        boxrule=0.8pt,        
        arc=4pt,              
        left=8pt, right=8pt, top=8pt, bottom=8pt, 
        fontupper=\small
    ]

    \textbf{Question:} Design a memorable slogan for a new line of eco-friendly diapers that appeals to environmentally conscious parents between the ages of 25 and 40, highlighting both sustainability and comfort.\\[3pt]

    \textbf{Hint ($\delta = +0.031$):} Focus on creating a slogan that balances the environmental benefits of the product with its user-friendly convenience, using language that evokes positive emotions and a connection to nature without being overly technical or abstract.\\[3pt]

    \textbf{$a_{\mathrm{hard}}$ \textit{(no-hint response)}:}\\
    Certainly! Here are a few slogan ideas for an eco-friendly diaper line targeting environmentally conscious parents aged 25-40:

    "Soft Comfort, Pure Earth."\\
    "Nature's Touch, Gentle on Bums."\\
    "Sustainable Softness for Little Ones."\\
    "Eco-Friendly, Eco-Friendly, and Still Gentle."\\
    "Green Comfort, Happy Baby, Happy Planet."\\
    "Sustainable Solutions for Sensible Parents."\\
    "Earth-Friendly, Baby-Friendly, Always Gentle."\\
    "Pure Comfort, Pure Planet."\\
    "Nature's Embrace, Your Baby's Best Choice."\\
    "Eco-Conscious, Comfort-Aware, Baby's Delight."\\[3pt]

    \textbf{$a_{\mathrm{assisted}}$ \textit{(hint-conditioned response)}:}\\
    Here's a potential slogan:

    "Soft on Skin, Kind to Earth - EcoGentle Diapers"

    This slogan aims to balance the environmental benefits of the product ("Kind to Earth") with its user-friendly convenience ("Soft on Skin"). The phrase "EcoGentle" creates a positive emotional association with nature while also conveying the gentle, comfortable nature of the diapers. It's simple, memorable, and conveys a clear message that resonates with environmentally conscious parents.

    \end{tcolorbox}
    \vspace{-0.1in}
    \caption{Second example. The hint asks for a slogan balancing sustainability with comfort using emotionally resonant language. $a_{\mathrm{hard}}$ throws ten generic options without commentary; $a_{\mathrm{assisted}}$ commits to a single slogan with a brand name (\emph{EcoGentle}) and explains why it satisfies the brief.}
    \label{fig:example_slogan}
    \vspace{-0.1in}
\end{figure*}

\end{document}